\documentclass[twoside,leqno,twocolumn]{article}

\usepackage[letterpaper]{geometry}

\usepackage{ltexpprt}
\usepackage{hyperref}

\usepackage[ruled, vlined, nofillcomment, linesnumbered]{algorithm2e}

\usepackage{amsmath}
\usepackage{amssymb}
\usepackage{booktabs}
\usepackage{tikz}
\usepackage{subcaption}
\usepackage{pgfplots}
\usepackage{pgfplotstable}
\usetikzlibrary{shapes.geometric, intersections}
\usepackage{etoolbox}

\usepackage[numbers,sort&compress]{natbib}
\newcommand{\qed}{}

\makeatletter
\patchcmd{\algocf@makecaption@ruled}{\hsize}{\columnwidth}{}{} % Caption to stretch full text width
\patchcmd{\@algocf@start}{-1.5em}{0em}{}{} % For // to right margin
\makeatother

\newcommand{\set}[1]{\left\{#1\right\}}
\newcommand{\pr}[1]{\left(#1\right)}
\newcommand{\fpr}[1]{\mathopen{}\left(#1\right)}

\newcommand{\abs}[1]{{\left|#1\right|}}
\newcommand{\floor}[1]{{\left\lfloor#1\right\rfloor}}
\newcommand{\ceil}[1]{{\left\lceil#1\right\rceil}}

\newcommand{\define}{\leftarrow}

\DeclareRobustCommand{\dispfunc}[2]{%
	\ensuremath{%
		\ifthenelse{\equal{#2}{}}%
			{\mathit{#1}}%
			{\mathit{#1}\fpr{#2}}}}

\newcommand{\ent}[1]{\dispfunc{H}{#1}}
\newcommand{\gini}[1]{\dispfunc{G}{#1}}
\newcommand{\gain}[1]{\dispfunc{I}{#1}}

\newcommand{\ind}[1]{\dispfunc{ind}{#1}}
\newcommand{\bern}[1]{\dispfunc{Bern}{#1}}

\newcommand{\bigO}[1]{\dispfunc{\mathcal{O}}{#1}}

\newcommand{\dtname}[1]{\textsl{#1}}

\newcommand{\dtsparse}{\dtname{sparse}\xspace}

\newcommand{\algupdsparse}{\textsc{UpdEnt}\xspace}
\newcommand{\algupdgini}{\textsc{UpdGini}\xspace}

\newcommand{\algbasesparse}{\textsc{BaseEnt}\xspace}
\newcommand{\algbasegini}{\textsc{BaseGini}\xspace}

\SetKw{Break}{break}
\SetKw{Output}{output}

\pgfdeclarelayer{background}
\pgfdeclarelayer{foreground}
\pgfsetlayers{background,main,foreground}

\definecolor{yafaxiscolor}{rgb}{0.3, 0.3, 0.3}

\definecolor{yafcolor1}{rgb}{0.4, 0.165, 0.553}
\definecolor{yafcolor2}{rgb}{0.949, 0.482, 0.216}
\definecolor{yafcolor3}{rgb}{0.47, 0.549, 0.306}
\definecolor{yafcolor4}{rgb}{0.925, 0.165, 0.224}
\definecolor{yafcolor5}{rgb}{0.141, 0.345, 0.643}
\definecolor{yafcolor6}{rgb}{0.965, 0.933, 0.267}
\definecolor{yafcolor7}{rgb}{0.627, 0.118, 0.165}
\definecolor{yafcolor8}{rgb}{0.878, 0.475, 0.686}

\newlength{\yafaxispad}
\newlength{\yaftlpad}
\newlength{\yaflabelpad}
\newlength{\yafaxiswidth}
\newlength{\yafticklen}
\makeatletter
\def\pgfplots@drawtickgridlines@INSTALLCLIP@onorientedsurf#1{}
\makeatother

\newcommand{\yafdrawaxis}[4]{
	\pgfplotstransformcoordinatex{#1}\let\xmincoord=\pgfmathresult 
	\pgfplotstransformcoordinatex{#2}\let\xmaxcoord=\pgfmathresult 
	\pgfplotstransformcoordinatey{#3}\let\ymincoord=\pgfmathresult 
	\pgfplotstransformcoordinatey{#4}\let\ymaxcoord=\pgfmathresult 
	\pgfsetlinewidth{\yafaxiswidth} 
	\pgfsetcolor{yafaxiscolor}
	\pgfpathmoveto{\pgfpointadd{\pgfpointadd{\pgfplotspointrelaxisxy{0}{0}}{\pgfqpointxy{\xmincoord}{0}}}{\pgfqpoint{-0.5\yafaxiswidth}{\yafaxispad}}}
	\pgfpathlineto{\pgfpointadd{\pgfpointadd{\pgfplotspointrelaxisxy{0}{0}}{\pgfqpointxy{\xmaxcoord}{0}}}{\pgfqpoint{0.5\yafaxiswidth}{\yafaxispad}}}
	\pgfpathmoveto{\pgfpointadd{\pgfpointadd{\pgfplotspointrelaxisxy{0}{0}}{\pgfqpointxy{0}{\ymincoord}}}{\pgfqpoint{\yafaxispad}{-0.5\yafaxiswidth}}}
	\pgfpathlineto{\pgfpointadd{\pgfpointadd{\pgfplotspointrelaxisxy{0}{0}}{\pgfqpointxy{0}{\ymaxcoord}}}{\pgfqpoint{\yafaxispad}{0.5\yafaxiswidth}}}
	\pgfusepath{stroke}
}

\pgfplotscreateplotcyclelist{yaf}{% 
{yafcolor5,mark options={scale=0.75},mark=o}, 
{yafcolor2,mark options={scale=0.75},mark=square},
{yafcolor3,mark options={scale=0.75},mark=triangle},
{yafcolor4,mark options={scale=0.75},mark=o},
{yafcolor1,mark options={scale=0.75},mark=o},
{yafcolor6,mark options={scale=0.75},mark=o},
{yafcolor7,mark options={scale=0.75},mark=o},
{yafcolor8,mark options={scale=0.75},mark=o}} 

\pgfkeys{/pgf/number format/.cd,1000 sep={\,}}

\pgfplotsset{axis y line=left, axis x line=bottom,
	tick align=outside,
	tickwidth=\yafticklen,
	clip = false,
    x axis line style= {-, line width = 0pt, color=black!0},
    y axis line style= {-, line width = 0pt, color=black!0},
    x tick style= {line width = \yafaxiswidth, color=yafaxiscolor, yshift = \yafaxispad},
    y tick style= {line width = \yafaxiswidth, color=yafaxiscolor, xshift = \yafaxispad},
    x tick label style = {font=\small, yshift = \yaftlpad, inner xsep = 0pt},
    y tick label style = {font=\small, xshift = \yaftlpad},
    every axis y label/.style = {at = {(ticklabel cs:0.5)}, rotate=90, anchor=center, font=\small, yshift = -\yaflabelpad, inner sep = 0pt},
    every axis x label/.style = {at = {(ticklabel cs:0.5)}, anchor=center, font=\small, yshift = \yaflabelpad},
    x tick label style = {font=\small, yshift = 1pt},
    grid = major,
    major grid style  = {dash pattern = on 1pt off 3 pt},
	every axis plot post/.append style= {line width=\yafaxiswidth} ,
	legend cell align = left,
	legend style = {inner sep = 1pt, cells = {font=\scriptsize}},
	legend image code/.code={%
		\draw[mark repeat=2,mark phase=2,#1] 
		plot coordinates { (0cm,0cm) (0.15cm,0cm) (0.3cm,0cm) };% 
	} 
}

\newcommand{\pgfprintduration}[1]{%
	\ifthenelse{\equal{#1}{}}{---}{%
	\pgfmathsetmacro{\minutes}{floor(#1 / 60)}%
	\pgfmathsetmacro{\seconds}{#1 - 60*\minutes}%
	\pgfmathifthenelse{\minutes > 0}{"\pgfmathprintnumber{\minutes}m \pgfmathprintnumber[fixed,precision=0]{\seconds}s"}{"\pgfmathprintnumber{\seconds}s"}\pgfmathresult}}

\makeatletter
\def\parsenode[#1]#2\pgf@nil{%
    \tikzset{label node/.style={#1}}
    \def\nodetext{#2}
}

\tikzset{
    add node at x/.style 2 args={
        name path global=plot line,
        /pgfplots/execute at end plot visualization/.append={
                \begingroup
                \@ifnextchar[{\parsenode}{\parsenode[]}#2\pgf@nil
            \path [name path global = position line #1-1]
                ({axis cs:#1,0}|-{rel axis cs:0,0}) --
                ({axis cs:#1,0}|-{rel axis cs:0,1});
            \path [xshift=1pt, name path global = position line #1-2]
                ({axis cs:#1,0}|-{rel axis cs:0,0}) --
                ({axis cs:#1,0}|-{rel axis cs:0,1});
            \path [
                name intersections={
                    of={plot line and position line #1-1},
                    name=left intersection
                },
                name intersections={
                    of={plot line and position line #1-2},
                    name=right intersection
                },
                label node/.append style={pos=1}
            ] (left intersection-1) -- (right intersection-1)
            node [label node]{\nodetext};
            \endgroup
        }
    },
    add node at y/.style 2 args={
        name path global=plot line,
        /pgfplots/execute at end plot visualization/.append={
                \begingroup
                \@ifnextchar[{\parsenode}{\parsenode[]}#2\pgf@nil
            \path [name path global = position line #1-1]
                ({axis cs:0,#1}-|{rel axis cs:0,0}) --
                ({axis cs:0,#1}-|{rel axis cs:1,1});
            \path [yshift=1pt, name path global = position line #1-2]
                ({axis cs:0,#1}-|{rel axis cs:0,0}) --
                ({axis cs:0,#1}-|{rel axis cs:1,1});
            \path [
                name intersections={
                    of={plot line and position line #1-1},
                    name=left intersection
                },
                name intersections={
                    of={plot line and position line #1-2},
                    name=right intersection
                },
                label node/.append style={pos=1}
            ] (left intersection-1) -- (right intersection-1)
            node [label node] {\nodetext};
            \endgroup
        }
    }
}
\makeatother

\begin{document}

\title{Approximating splits for decision trees quickly in sparse data streams\thanks{This research is supported by the Academy of Finland projects MALSOME (343045)}}

\author{Nikolaj Tatti\thanks{HIIT, University of Helsinki, nikolaj.tatti@helsinki.fi}}

\date{}

\maketitle              % typeset the header of the contribution

\fancyfoot[R]{\scriptsize{Copyright \textcopyright\ 2025 by SIAM\\
Unauthorized reproduction of this article is prohibited}}

\begin{abstract}
\small\baselineskip=9pt
Decision trees are one of the most popular classifiers in the machine learning literature.
While the most common decision tree learning
algorithms treat data as a batch, numerous algorithms have been proposed
to construct decision trees from a data stream.
A standard training strategy involves augmenting the current tree by changing a leaf node into a split.
Here we typically maintain counters in each leaf which allow us to determine the optimal split, and whether the split
should be done. 
In this paper we focus on how to speed up the search for the optimal split when dealing with sparse binary features and a binary class.
We focus on finding splits that
have the
approximately optimal information gain or Gini index. 
In both cases finding the optimal split can be done in $\bigO{d}$ time,
where $d$ is the number of features.
We propose an algorithm that yields $(1 + \alpha)$ approximation when using conditional entropy in amortized $\bigO{\alpha^{-1}(1 + m\log d) \log \log n}$
time, where $m$ is the number of 1s in a data point, and $n$ is the number of data points.
Similarly, for Gini index, we achieve $(1 + \alpha)$ approximation in amortized $\bigO{\alpha^{-1} + m \log d}$ time.
Our approach is beneficial for sparse data where $m \ll d$.
In our experiments
we find almost-optimal splits efficiently, faster than the baseline, overperforming the theoretical approximation guarantees.
\end{abstract}

\section{Introduction}

Decision trees are one of the most popular classifiers in the machine learning literature.
While the most well-known decision tree learning
algorithms~\citep{quinlan1986induction,quinlan2014c4,breiman2017classification}
treat data as a batch, numerous algorithms have been proposed
to construct decision trees from a data stream.
In the streaming setup points arrive one-by-one, and the tree is updated on the fly.
A core step in almost every such learner is to determine the optimal split after a new data point.
A baseline approach for finding the optimal split is to test every feature, and by maintaining certain counters
this can be done in $\bigO{d}$ time, where $d$ is the number of features.

In this paper we consider two scenarios where we can speed up this step.
We focus on binary features and binary labels, and our goal is to speed up selecting features
when the features are sparse. Such case can occur when dealing for example with text data,
where the dictionary of words can be large while the actual inputs, for example, bag-of-word representations
of sentences, do not contain many active features. 

First, we consider optimizing information gain. 
We propose an algorithm that---after adding a data
point---finds a $(1 + \alpha)$-approximate optimal feature in $\bigO{\alpha^{-1} (1 + m \log d) \log \log n}$
amortized time, where $m$ is the number of 1s in a data point.
This speeds up the search if $m \ll d$ which is the case for sparse data.

Roughly speaking, the approach is based on grouping features with similar
probabilities $p(y = 1 \mid j = 0)$. That is, each group $i$ has an interval, say $B_i$,
and if $p(y = 1 \mid j = 0) \in B_i$, then feature $j$ will be stored in group $i$.
Features in each group are then stored in a search tree using a specific key.
To find the approximate feature we search each tree for the optimal value, and then select the
best feature among the candidates.

It turns out that we can only have  $\bigO{\alpha^{-1} \log \log n}$ such search trees,
and updating the key for a feature  $i$ is required only if we encounter a data point with $x_i = 1$.
This will eventually lead to the stated running time.

In the second setup,
we optimize Gini index.
We propose an algorithm  that
finds a $(1 + \alpha)$-approximate optimal feature in $\bigO{\alpha^{-1} + m \log d}$
amortized time. 

Our approach is similar to the approach for information gain with few technical differences.
Roughly speaking, when using information gain we do not need to delete
features from their outdated groups as $p(y = 1 \mid j = 0)$ is changing. The same does not hold
for Gini index: the features need to be removed from their old groups, potentially increasing the update time
significantly. Luckily, we can mitigate the issue by allowing the intervals $B_i$ associated with the groups to overlap.

\section{Related work}
\label{sec:related}

Numerous algorithms have been proposed to construct decision trees as data arrives in a stream.
\citet{schlimmer1986case} proposed an algorithm that monitors existing nodes
and the moment a split becomes suboptimal, the sub-tree is replaced with the new split.
\citet{utgoff1988id5} and \citet{utgoff1989incremental} proposed a variant of this approach, except that
the sub-tree is not thrown away but rather modified such that the split is updated. 

\citet{domingos2000mining}
proposed a popular choice for growing decision trees incrementally. The
algorithm ignores the previous data points when starting a new leaf, thus allowing to maintain only the label counters.
As a criterion, the authors propose to compare the best and the second best split, and used a Hoeffding bound
to decide whether to split. Unfortunately, \citet{rutkowski2012decision}
showed that the assumptions used to derive the bound were incorrect, and thus can only be used as a
heuristic. Moreover, the authors propose a more conservative bound based on McDiarmid's inequality.

\citet{manapragada2018extremely,manapragada2022eager}
introduced a variant that also monitors the non-leaf nodes, and if another
split becomes optimal (based on comparing the best and the second best feature) then the current sub-tree is
replaced with the new split.

\citet{jin2003efficient} proposed a test by estimating the entropy difference
with a normal distribution. A similar approach was also considered by~\citet{gratch1995sequential}.

\citet{bressan2023fully} and \citet{bressan2024dynamic} proposed maintaining decision trees
in a data stream within an \emph{additive} error.  Here, the goal is orthogonal to
ours: the authors' goal is to maintain a decision tree, done by reducing the number of split updates, while
our goal is to---under certain conditions---speed up the split. The algorithms proposed by the authors are linear
w.r.t. the number of features. Combining their approach with ours with the goal of obtaining a sublinear
update time is an interesting line of future work.

A popular related line of work is building decision trees while the underlying data stream distribution is changing.
\citet{hulten2001mining}
proposed an algorithm that adapts a decision tree to a changing data stream.
Here, the decision tree is based on data in a sliding window. If a non-leaf node becomes
suboptimal, an alternative tree is grown from that node, possibly replacing the current sub-tree.
Alternative approaches, where a user does not need to specify the window size, were considered by~\citet{bifet2009adaptive}.
Moreover, \citet{bifet2009new} proposed constructing an ensemble of trees of increasing sizes in order to adapt better to changes.

\section{Preliminary notation}
\label{sec:prel}

Assume that we are given a multiset $D$ of pairs $(x, y)$, where $x$ are the features 
and $y$ is the class of the sample.
Assume that features are binary and define
the counters 
\[
	n(D) = \abs{D} \quad\text{and}\quad
	n_{iv}(D) = \abs{\set{(x, y) \in D \mid x_i = v}}\quad.
\]
We also define the counters associated with each label, 
\[
\begin{split}
	c_k(D) & = \abs{\set{(x, y) \in D \mid y = k}} \quad\text{and}\quad \\
	c_{ivk}(D)  & = \abs{\set{(x, y) \in D \mid x_i = v, y = k}}\ .
\end{split}
\]

If $D$ is known from the context, we will often drop $D$ from the notation and write
$n$, $n_{iv}$, $n_{ivk}$, and $c_k$.

The entropy and the conditional entropy of a feature $i$ are defined as
\[
\begin{split}
	\ent{D} & = -\sum_{k} \frac{c_k(D)}{n(D)} \log \frac{c_k(D)}{n(D)}  \ \text{ and }\  \\
	\ent{D \mid i}  & = -\sum_{v} \sum_k \frac{c_{ivk}(D)}{n(D)} \log \frac{c_{ivk}(D)}{n_{iv}(D)}, 
\end{split}
\]
where the standard notation $0 \times \log 0 = 0$ applies.

A common criterion for the goodness of a feature is the information gain,
$\gain{i, D} = \ent{D} - \ent{D \mid i}$. Thus, 
finding the optimal split is equal to finding a feature with the smallest conditional entropy $\ent{D \mid i}$.

We also consider a popular alternative, namely a Gini index. The Gini index $\gini{D}$ and the conditional Gini index $\gini{D \mid i}$ 
are defined as
\[
\begin{split}
	\gini{D} & = \sum_{k} \frac{c_k(D)}{n(D)}\bigg(1 - \frac{c_k(D)}{n(D)}\bigg) \ \text{ and }  \\
	\gini{D \mid i}  & = \sum_{v, k} \frac{c_{ivk}(D)}{n(D)} \bigg(1 - \frac{c_{ivk}(D)}{n_{iv}(D)}\bigg)\ .
\end{split}
\]
When using Gini index, the best feature for a split is the one that minimizes $\gini{D \mid i}$. 

\section{Approximate splits when using entropy}
\label{sec:sparse}

In this section we provide a method that approximates the best binary feature for the split quickly.
We will assume throughout this section that the output variable $y$ is binary. 
We should stress that this assumption is crucial.
For notational simplicity, we also assume that all our features are binary. Naturally, if we have
features that are binary, categorical, or numerical, we can combine different search strategies,
and select the best out of three.

Let us first define 
\begin{equation}
\label{eq:fixent}
\begin{split}
	\ent{D \mid i, \theta} = &
		-\frac{c_{i00}}{n} \log (1 - \theta) - \frac{c_{i01}}{n} \log \theta  \\
		 & -\frac{c_{i10}}{n} \log \frac{c_{i10}}{n_{i1}} - \frac{c_{i11}}{n}  \log \frac{c_{i11}}{n_{i1}}\ .
\end{split}
\end{equation}
Note that $\ent{D \mid i, \theta} \geq \ent{D \mid i, c_{i01} / n_{i0}} = \ent{D \mid i}$, since the first two terms in Eq.~\ref{eq:fixent} form a negative,
scaled log-likelihood of a Bernoulli model, for which the optimal parameter is $c_{i01}/{n_0}$.

\begin{proposition}
\label{prop:decomp}
The conditional entropy can be decomposed into two parts,
\[
	\ent{D \mid i, \theta} = \frac{1}{n} \pr{C(c_0, c_1, \theta) + K(c_{i10}, c_{i11}, \theta)},
\]
where
\[
\begin{split}
	C(c_0, c_1, \theta) & = -c_{0} \log (1 - \theta) - c_{1} \log \theta, \quad\text{and} \\
	K(c_{i10}, c_{i11}, \theta) & = -c_{i10} \log \frac{c_{i10}}{n_{i1}(1 - \theta)} - c_{i11}  \log \frac{c_{i11}}{n_{i1}\theta}\quad.
\end{split}
\]
\end{proposition}

Note that the definition of $K$ depends on $n_{i1}$ but $n_{i1} = c_{i10} + c_{i11}$, so $K$ is well-defined. 

\begin{proof}
The claim follows immediately by using the identities $c_{i00} = c_0 - c_{i10}$ and $c_{i01} = c_1 - c_{i11}$ 
together with Eq.~\ref{eq:fixent}, and rearranging the terms.
\end{proof}

Note that, for a \emph{fixed} $\theta$, the term $C(\cdot)$ does not depend on the feature $i$.
The term $K(\cdot)$ depends only on $c_{i10}$ and $c_{i11}$ (and $\theta$),
changes only for a feature $i$ when we process a data point $(x, y)$ with $x_i = 1$.

Our high-level strategy is as follows:
we construct a set of real numbers $\set{\mu_i}$, and to each $\mu_i$ we assign a set of features. We then find a feature with the smallest $K(\cdot, \cdot, \mu_i)$, and
among these candidates we select the feature $j$ with the smallest $\ent{D \mid j}$.

To be more specific, assume $\alpha > 0$ and let $D$ be a dataset with $n$ data points.
We define a set of intervals $B_{-1}, B_{-2}, \ldots, B_{-\ell}$ as
\[
    B_i = [s_i, t_i], \ \text{where}\ 
    s_i = 2^{-(1 + \alpha)^{-i}}, \ 
    t_i = 2^{-(1 + \alpha)^{-i-1}}, 
\]
for $i = -1, \ldots, -\ell$,
and $\ell$ is set such that $s_{-\ell} \leq 1/n < t_{-\ell}$. We also set
$B_{-\infty} = [s_{-\infty}, t_{-\infty}] = [0, 0]$.
We extend the definition to $B_i$ for $i = 1, \ldots, \ell$ by setting $B_i = [s_i, t_i]$, where $s_i = 1 - t_{-i}$ and $t_i = 1 - s_{-i}$.
Finally, for each bin $B_i$, we define a centroid $\mu_i = s_i$ for $i < 0$ and
$\mu_i = t_i$ for $i > 0$.

\begin{proposition}
\label{prop:sparsebound}
Let $\alpha$, $D$, $n$, $\set{B_i}$, and $\ell$ as defined above.
Let $j$ be a feature and let $\rho = c_{j01} / n_{j0}$.
For any $B_i$ such that $\rho \in B_i$, we have $\ent{D \mid j, \mu_i} \leq (1 + \alpha) \ent{D \mid j}$.
Moreover, there is $B_i$ such that $\rho \in B_i$.
The number of intervals $\abs{\set{B_i}}$ is in
$\bigO{\alpha^{-1}\log\log n}$. 
\end{proposition}

\begin{proof}
Let us first prove the first claim.
Note that $\ent{D \mid j} = \ent{D \mid j, \rho}$.
If $\rho = 0$,  then $\rho \in B_{-\infty}$ and $B_{-\infty}$ is the only one containing $\rho$. Moreover, $\mu_{-\infty} = \rho$ so the claim follows.
The case for $\rho = 1$ is similar.

Assume that $0 < \rho <  1$. Then $1/n \leq \rho \leq 1 - 1/n$, and thus there is $B_i$ such that $\rho \in B_i$.

Assume that $i < 0$. Then $\mu_i = s_i \leq \rho$, and $\rho \leq t_i = \mu_i^{(1 + \alpha)^{-1}}$ implying $\rho^{1 + \alpha} \leq \mu_i$.
Let $A$ be the last two terms in Eq.~\ref{eq:fixent}.
Then
\[
\begin{split}
    \ent{D \mid j, \mu_i} & = A -\frac{c_{j00}}{n} \log (1 - \mu_i) - \frac{c_{j01}}{n} \log \mu_i \\
    & \leq A -\frac{c_{j00}}{n} \log (1 - \rho) - (1 + \alpha)\frac{c_{j01}}{n} \log \rho \\
    & \leq (1 + \alpha)(A -\frac{c_{j00}}{n} \log (1 - \rho) - \frac{c_{j01}}{n} \log \rho ) \\
    & = (1 + \alpha)\ent{D \mid j, \rho},
\end{split}
\]
where the second inequality follows since $A$ and $-\frac{c_{j00}}{n} \log (1 - \rho)$ are both positive.

Assume $i > 0$.
Then, similarly,
$1 - \mu_i \leq 1 - \rho$ and $(1  - \rho)^{1 + \alpha} \leq 1 - \mu_i$. Thus
\[
\begin{split}
    \ent{D \mid j, \mu_i} & = A -\frac{c_{j00}}{n} \log (1 - \mu_i) - \frac{c_{j01}}{n} \log\mu_i \\
    & \leq A -(1 + \alpha)\frac{c_{j00}}{n} \log (1 - \rho) - \frac{c_{j01}}{n} \log \rho \\
    & \leq (1 + \alpha)\ent{D \mid j, \rho},
\end{split}
\]
proving the first claim.

To continue, note that
since $t_{-\ell} > 1/n$, we have $(1 + \alpha)^{\ell} > (1 + \alpha) \log_2 n$.
Solving for $\ell$ leads to 
\[
	\ell < 1 + \log_{1 + \alpha} \log_2 n \in 
	\bigO{\alpha^{-1} \log\log n }.
\]
There are $2(1 + \ell) \in \bigO{\alpha^{-1} \log \log n }$ intervals,
proving the second claim.
\qed
\end{proof}

Proposition~\ref{prop:sparsebound} provides us with an algorithm: for each $\mu_i$
we maintain a balanced search tree, say $T_i$, where features are stored with a
key of $K(c_{i10}, c_{i11}, \mu_i)$.  We guarantee that each feature $j$ is
stored in $T_i$ such that $\rho = \frac{c_{j01}}{n_{j0}} \in B_i$.  From each
tree, we select a candidate feature with the smallest $K(\cdot)$, and among
these candidates we select the feature with the smallest conditional entropy.

The issue with this approach is maintaining the invariant that $\rho \in B_i$.
If we were to require that a feature $j$ can be stored only in $T_i$, then moving
the feature $j$ between neighboring bins may become too computationally expensive.

Instead, we will allow storing the same feature in multiple trees. We will see later that this does not change the approximation guarantee.
Our algorithm is as follows.
\algupdsparse (see Algorithm~\ref{alg:updsparse}) maintains the following invariant:
each feature $j$ with $\rho_j = \frac{c_{j01}}{n_{j0}}$, that has occurred at least once, has two indices $a(j)$ and $b(j)$
such that there is $i$ with $a(j) \leq i \leq b(j)$ and $\rho_j \in B_i$.
The feature $j$ is then stored in the trees $T_{a(j)}, \ldots, T_{b(j)}$ (excluding $T_0$ as $B_0$ is not defined).

In order to maintain the invariant, whenever we process data point $(x, y)$ with $x_j = 1$, \algupdsparse removes $j$ from any 
current trees, and adds $j$ to a single $T_i$ such that $\rho_j \in B_i$.
In order to find the correct $T_i$, we define, 
\begin{equation*}
	\ind{\rho} =
	\begin{cases}
		-\infty, & \text{if } \rho = 0, \\
		-\ceil{\log_{1 + \alpha} \log_{1/2} \rho}, & \text{if } 0 < \rho < 1/2, \\
		1 & \text{if } \rho = 1/2, \\
		\ceil{\log_{1 + \alpha} \log_{1/2} (1 -  \rho)}, & \text{if } 1/2 < \rho < 1, \\
		\infty, & \text{if } \rho = 1\quad. \\
	\end{cases}
\end{equation*}
A straightforward calculation shows that $\rho_j \in B_{\ind{\rho_j}}$.

We also need to see if the invariant, that is the existence of $i$ such that $a(j) \leq i \leq b(j)$ and $\rho_j \in B_i$,  holds for any variable with $x_j = 0$.
Since $j$ is currently stored in $T_{a(j)}, \ldots, T_{b(j)}$,
the invariant is violated if and only if
	$\rho_j < s_{a(j)}$ or
	$\rho_j > t_{b(j)}$.
Unfortunately, we cannot check every feature as this would be too costly.
Instead we use the following proposition.

\begin{proposition}
\label{prop:sparsecheck}
For a binary feature $j$, we have that
$c_{j01}/n_{j0} \lessgtr  \theta$ if and only if $\theta n_{j1} - c_{j11} \lessgtr \theta n - c_1$.
\end{proposition}

\begin{proof}
The claim follows immediately from the identities $c_{j01} + c_{j11} = c_1$
and $n_{j0} + n_{j1} = n$.
\end{proof}

We use Proposition~\ref{prop:sparsecheck} as follows.
We maintain 2 addional sets of trees.
We store each feature $j$ in a search tree $L_{a(j)}$
with a key $s_{a(j)} n_{j1} - c_{j11}$.
We store each feature $j$ in a search tree $U_{b(j)}$
with a key $t_{b(j)} n_{j1} - c_{j11}$.

In order to maintain the invariant we test every tree $L_i$
and extract any feature $j$ whose key is smaller
than $s_i n - c_1$. Proposition~\ref{prop:sparsecheck} states that such feature violates the invariant
so we lower $a(j)$ and update the corresponding trees.
We may need to lower $a(j)$ several times for a single feature.

Similarly, we test every tree $U_i$
and extract any feature $j$ whose key is larger
than $t_i n - c_1$. 
We increase $b(j)$, possibly several times, updating the trees.

In order to approximate the best feature we select the best item from each $T_j$ and select
the best feature among the candidates.

In order to keep the number of trees small, we do the following practical optimizations.
Search trees $T_{-\infty}$ and $T_{\infty}$ are handled differently.
If for some feature $a(j) = b(j) = \infty$ but $\rho_j \notin B_{\infty}$,
then instead of lowering $a(j)$ we update \emph{both} $a(j) = b(j) = i$ such that $\rho_j \in B_i$.
We do similar update if
$a(i) = b(i) = -\infty$ but $\rho_j \notin B_{-\infty}$.
In order to keep the number of trees small, we prune the empty trees. 

\begin{algorithm}[t!]
\caption{$\algupdsparse(x, y)$, processes data point $(x, y)$ and approximates the best split.} 
\label{alg:updsparse}
update $n$, $c_0$, and $c_1$\;
\ForEach {$j$ with $x_j = 1$} {
	update $c_{j11}$, $c_{j10}$, and $n_{j1}$\;
	$i \define \ind{(c_1 - c_{j11})/(n - n_{j1})}$\;
	$a(j) \define i$, $b(j) \define i$, and update the trees\;
}
\ForEach {$j \in U_{-\infty}$ with $c_{j11} < c_1$} {
	$i \define \ind{(c_1 - c_{j11})/(n - n_{j1})}$\;
	$a(j) \define i$, $b(j) \define i$, and update the trees\;
}
\ForEach {$j \in L_{\infty}$ with $n_{j1} - c_{j11} > n - c_1$} {
	$i \define \ind{(c_1 - c_{j11})/(n - n_{j1})}$\;
	$a(j) \define i$, $b(j) \define i$, and update the trees\;
}

\ForEach {$L_i$ in descending order} {
	\ForEach {violating $j \in L_i$ }{ 
		$a(j) \define a(j) - 1$, and update the trees\; \label{alg:extend1}
	}
}
\ForEach {$U_i$ in ascending order} {
	\ForEach {violating $j \in U_i$ }{ 
		$b(j) \define b(j) + 1$, and update the trees\; \label{alg:extend2}
	}
}
delete empty $T_i$, $U_i$, and $U_i$\;
$J \define$ entries with the smallest key in each $T_i$\;
\Return $\arg \min_{j \in J} \ent{D \mid j}$\;
\end{algorithm}

\begin{proposition}
\label{prop:sparsetime}
Let $d$ be the number of features, $n$ be the number of data points,
and $m$ be the average number of 1s in a data point.
Let $\alpha > 0$. 
Then \algupdsparse yields $(1 + \alpha)$ approximation in $\bigO{\alpha^{-1} (1 + m \log d) \log \log n}$ amortized time.
\end{proposition}

\begin{proof}
To prove the approximation guarantee,
let $j$ be the optimal feature, and let $\rho = \frac{c_{j01}}{n_{j0}}$.
Proposition~\ref{prop:sparsecheck} guarantees that the invariant is maintained,
and thus there is $B_i$ such that $\rho \in B_i$ and $j$ is stored in $T_i$.

Proposition~\ref{prop:sparsebound} states that $\ent{D \mid j, \mu_i} \leq (1 + \alpha) \ent{D \mid j}$.
Let $j'$ be the feature in $T_i$ with the smallest key, then Proposition~\ref{prop:decomp}
implies that $\ent{D \mid j', \mu_i} \leq \ent{D \mid j, \mu_i}$.
Finally, the algorithm returns $j^*$  with
\[
\begin{split}
	\ent{D \mid j^*} & \leq \ent{D \mid j'} 
	 \leq \ent{D \mid j', \mu_i}  \\
	 & \leq \ent{D \mid j, \mu_i} 
	 \leq (1 + \alpha) \ent{D \mid j}
\end{split}
\]
proving the approximation guarantee.

To prove the running time we will show that the time needed to add $D$ with $n$
points from scratch is in 
\[
	\bigO{\frac{1}{\alpha} (n + \sum_{x \in D} \sum_i x_i \log d) \log \log n}.
\]

Let $h$ be the number of trees $T_i$. 
Proposition~\ref{prop:sparsebound} states that $h \in \bigO{\alpha^{-1}  \log \log n}$.

Let us first consider the first for-loop. This for-loop is evaluated $\bigO{\sum_{x \in D, i} x_i}$,
times, and each iteration requires $\bigO{h \log d}$ as we need to update $\bigO{h}$ search trees and each tree
has $\bigO{d}$ elements.

Let us now consider the next four for-loops.
Let $\Delta_{j}$ be the number of times a feature $j$ is updated in these for-loops.
First, note that 
once a feature leaves $T_{-\infty}$ or $T_{\infty}$ it will never return back, that is, the 2nd and 3rd for-loops are evaluated at most only once per feature.
Moreover, each evaluation of $j$ in the last two for-loops increases the number of trees containing $j$.
Consequently, there can be at most $\bigO{h}$ evaluations for a single $j$ in these for-loops \emph{between} the updates to $j$ in the first for-loop,
or in other words, $\Delta_j \in \bigO{h\sum_{x \in D} x_j}$. Each update requires $\bigO{\log d}$ time.
In summary, adding $n$ data points requires $\bigO{hn + h\sum_j \Delta_j \log d + h\sum_{x \in D, j} x_j \log d } \subseteq \bigO{hn + h\sum_{x \in D, j} x_j \log d }$ time.

Finally, the last three lines require $\bigO{h}$ time, proving the proposition.
\qed
\end{proof}

\section{Approximate splits when using Gini index}
\label{sec:gini}

Our strategy for selecting a feature with an approximately best Gini index is similar.
As before we assume that we are dealing with binary features and binary labels. 
Let
\[
	\gini{D \mid j, \theta} = 2\frac{c_{j10}}{n} \frac{c_{j11}}{n_{j1}} + \frac{2}{n}
	\begin{cases}
	c_{j00} \theta & \theta \geq 1/2 \\
	c_{j01} (1 - \theta) & \theta < 1/2 \\
	\end{cases}
	\quad.
\]

Note that $\gini{D \mid j, c_{j01} / n_{j0}} = \gini{D \mid j}$.
Similar to Proposition~\ref{prop:decomp}, we can decompose $\gini{D \mid j, \theta}$
into two components, one not depending the feature $j$.

\begin{proposition}
\label{prop:decompgini}
The Gini index can be written as
\[
	\gini{D \mid j, \theta} = \frac{1}{n} \pr{C(c_0, c_1, \theta) + K(c_{j10}, c_{j11}, \theta)},
\]
where
\[
\begin{split}
	C(c_0, c_1, \theta) & =
	\begin{cases}
	2c_{0}\theta & \theta \geq 1/2 \\
	2c_{1}(1 -\theta) & \theta < 1/2 \\
	\end{cases}
\quad\text{and} \\
	K(c_{j10}, c_{j11}, \theta) & =
	2c_{j10} \frac{c_{j11}}{n_{j1}} - 
	\begin{cases}
		2c_{j10}\theta & \theta \geq 1/2 \\
		2c_{j11}(1 -\theta) & \theta < 1/2\ . \\
    \end{cases}
\end{split}
\]
\end{proposition}

As with the entropy, we define a set of intervals $\set{B_i}$ and centroids $\set{\mu_i}$,
and estimate the Gini index of a feature $j$ such that $\frac{c_{j01}}{n_{j0}} \in B_i$ with $\gini{D \mid j, \mu_i}$.

More specifically, assume that we are given $\alpha > 0$.
Define $\beta = \frac{\alpha}{\alpha + 2}$.
Define the intervals $B_0, \ldots, B_\ell$ as
\[
\begin{split}
    B_i & = [s_i, t_i], \quad\text{where}\quad \\
    s_i & = (i/2 - 1/4) \beta,\ 
    t_i = (i/2 + 3/4) \beta,\ 
    i = 0, \ldots, \ell,
\end{split}
\]
where $\ell$ is set such that $s_\ell < 1 \leq t_\ell$.
With each $B_i$ we associate a centroid, $\mu_i = (s_i + t_i)/2 = (i/2 + 1/4)\beta$.
Note that the intervals are overlapping, and shifted by $1/4$ to the left.
Both properties will be needed later in Proposition~\ref{prop:ginitime} when proving the running time.

Next we show that we can use $\gini{D \mid j, \mu_i}$ to approximate $\gini{D \mid j}$.

\begin{proposition}
\label{prop:ginibound}
Let $\set{B_i}$ and $\set{\mu_i}$ as defined above. 
Assume two features $j$ and $k$ such that $\frac{c_{j01}}{n_{j0}}, \frac{c_{k01}}{n_{k0}} \in B_i$ for some $B_i$.
If $\gini{D \mid j, \mu_i} \leq \gini{D \mid k, \mu_i}$, then $\gini{D \mid j} \leq (1 + \alpha)\gini{D \mid k}$.
\end{proposition}

\begin{proof}
Assume $\mu_i \geq 1/2$. Then $s_i \geq 1/2 - \beta/2 = (\alpha + 2)^{-1}$.
Given a feature $a$, define
$\rho_a = \frac{c_{a01}}{n_{a0}}$ and
$A_a = 2\frac{c_{a10}}{n} \frac{c_{a11}}{n_{a1}}$.

Let $\Delta_1 = \max(\rho_j / \mu_i, 1)$
and $\Delta_2 = \max(\mu_i / \rho_k, 1)$.
We have
\[
\begin{split}
	\gini{D \mid j} & =  A_j + 2\frac{c_{j00}}{n}\rho_j \leq A_j + \Delta_1 2\frac{c_{j00}}{n}\mu_i \\
	& \leq \Delta_1 (A_j + 2\frac{c_{j00}}{n}\mu_i) = \Delta_1 \gini{D \mid j, \mu_i}
\end{split}
\]
and similarly,
$\gini{D \mid k, \mu_i} \leq \Delta_2  \gini{D \mid k}$, leading to
\[
\begin{split}
	\gini{D \mid j} & \leq \Delta_1 \gini{D \mid j, \mu_i} \\
	& \leq \Delta_1 \gini{D \mid k, \mu_i} \leq  \Delta_1 \Delta_2 \gini{D \mid k} \quad.
\end{split}
\]
A straightforward calculation leads to 
\[
	\Delta_1\Delta_2 \leq \frac{\max(\mu_i, \rho_j)}{\min(\mu_i, \rho_k)} \leq \frac{t_i}{s_i} = 1 + \frac{\alpha}{(2 + \alpha)s_i} \leq 1 + \alpha,
\]
proving the claim for $\mu_i \geq 1/2$. The case for $\mu_i < 1/2$ is similar.
\qed
\end{proof}

Next we will describe \algupdgini, given in Algorithm~\ref{alg:updgini}.
As with \algupdsparse, for each interval $B_i$ we maintain 3 search trees, $T_i$, $U_i$, and $L_i$.
\algupdgini stores each feature, say $j$ with $\rho = \frac{c_{j01}}{n_{j0}}$ in a \emph{single} tree $T_i$ such that $\rho \in B_i$
with a key $K(c_{j10}, c_{j11}, \mu_i)$.
Moreover, 
we store each feature $j$ in a search tree $L_{i}$
with a key $s_{i} n_{j1} - c_{j11}$.
we store each feature $j$ in a search tree $U_{i}$
with a key $t_{i} n_{j1} - c_{j11}$.

We should point out that,
unlike with \algupdsparse, we must store each feature only in one tree, as otherwise we cannot guarantee the approximation.
This may be an issue if $\rho$ is at the border of two intervals, resulting in too many updates.
To counter this effect we have defined $B_i$ so that they overlap, so that enough points are required before we need to change the tree.

In order to keep the features in the correct trees with the correct keys,
we update every feature $j$ whenever we process a data point $x$ with $x_j = 1$.
Moreover, we use Proposition~\ref{prop:sparsecheck} to verify that the remaining features are also in the correct trees.
Once we determine that the feature needs updating we use the function
\begin{equation}
\label{eq:indgini}
	\ind{\rho} = \floor{2\rho(\alpha + 2)/{\alpha}}
\end{equation}
to find the index of a suitable $B_i$: it is straightforward to see that $\rho \in B_{\ind{\rho}}$.

\begin{algorithm}[t!]
\caption{$\algupdgini(x, y)$, processes data point $(x, y)$ and approximates the best split.} 
\label{alg:updgini}
update $n$, $c_0$, and $c_1$\;
\ForEach {$j$ with $x_j = 1$} {
	update $c_{j11}$, $c_{j10}$, and $n_{j1}$\;
	$i \define \ind{(c_1 - c_{j11})/(n - n_{j1})}$\;
	move $j$ to $T_i$, $L_i$ and $U_i$\;
}

\ForEach {violating $j \in L_i$ (or $j \in U_i$)}{  
	$k \define \ind{(c_1 - c_{j11})/(n - n_{j1})}$\;
	move $j$ to $T_k$, $L_k$ and $U_k$\; \label{alg:giniupdfeature}
}

delete empty $T_i$, $U_i$, and $U_i$\;
$J \define$ entries with the smallest key in each $T_i$\;
\Return $\arg \min_{j \in J} \gini{D \mid j}$\;
\end{algorithm}

\begin{proposition}
\label{prop:ginitime}
Let $d$ be the number of features and $n$ be the number of data points.
Assume we are adding a data point $(x, y)$. Let $m$ be the number of 1s in $x$.
Assume $\alpha > 0$. 
Then \algupdgini yields $(1 + \alpha)$ approximation in $\bigO{\alpha^{-1}  + m \log d}$ amortized time.
\end{proposition}

First we need the following lemma.

\begin{lemma}
\label{lem:viol}
Let $0 \leq x \leq y$ and  $0 \leq a \leq b$. 
Let $\beta = \abs{\frac{x + a}{y + b} - \frac{x}{y}}$. Then $b \geq \beta y$.
\end{lemma}

\begin{proof}
Assume $b < \beta y$. Then 
\[
\begin{split}
	\frac{x + a}{y + b} - \frac{x}{y} &
	\leq \frac{x + \beta y}{y + \beta y}  - \frac{x}{y} 
	 = \frac{\beta}{1 + \beta}\pr{1 - \frac{x}{y}}  < \beta
\end{split}
\]
and
\[
	\frac{x}{y} - \frac{x + a}{y + b} \leq \frac{x}{y} - \frac{x}{y + b} < \frac{x}{y} - \frac{x}{y + \beta y}
	= \frac{\beta}{1 + \beta}\frac{x}{y}  < \beta,
\]
contradicting the definition of $\beta$.
\qed
\end{proof}

\begin{proof}[of Proposition~\ref{prop:ginitime}]
To prove the approximation guarantee,
let $j$ be the optimal split variable and let $T_i$ be the tree, where $j$ is stored.
Let $j'$ be the feature in $T_i$ with the smallest key, then Proposition~\ref{prop:decompgini}
implies that $\gini{D \mid j', \mu_i} \leq \gini{D \mid j, \mu_i}$. Proposition~\ref{prop:ginibound}
implies that $\gini{D \mid j'} \leq (1 + \alpha) \gini{D \mid j}$.
Finally, the algorithm returns $j^*$  with
	$\gini{D \mid j^*} \leq \gini{D \mid j'} 
	 \leq (1 + \alpha) \gini{D \mid j}$
proving the approximation guarantee.

Next let us prove the computational complexity.
The first for-loop requires $\bigO{\sum_{x \in D, j} x_j \log d}$ total time to process $n$ data points.

To analyze the second for-loop,
select and fix a feature $j$. Let $k$ be the number of times the feature $j$ changes its bin due to the second for-loop. 
Let $u_\ell$ be the value of $n_{j0}$ after $j$ changes its bin for the $\ell$th time during the second for-loop.

Fix $\ell$ and let $i$ be the interval from which $j$ is moved when $j$ is moved for the $\ell$th time.
Let $\rho$ be the value of $c_{j01}/n_{j0}$ when $j$ is moved from $B_i$ and
let $\rho'$ be the earlier value of $c_{j01}/n_{j0}$ when $j$ is moved to $B_i$.

Now the definitions of $\ind{\cdot}$ and $s_i$ and $t_i$ imply
\[
	\abs{\rho' - \rho} \geq  \min(\rho' - s_i, t_i - \rho') \geq \gamma, \ \text{where}\  \gamma = \frac{\alpha}{4(2 + \alpha)}\ .
\]
Lemma~\ref{lem:viol} guarantees that $u_{\ell} \geq (1 + \gamma) u_{\ell - 1}$.
Applied iteratively, we have $n \geq u_k \geq (1 + \gamma)^k$, and consequently $k \in \bigO{\log_{1 + \gamma} n}$.
Since there are $\bigO{\alpha^{-1}}$ intervals, the second for-loop requires $\bigO{n/\alpha + d \log_{1 + \gamma}n \log d}$ time
to process $n$ data points.

Thus,
$\bigO{\sum_{x \in D, j} x_j \log d + n/\alpha + d \log_{1 + \gamma}n \log d}$
time is needed to process $n$ data points, or $\bigO{\alpha^{-1} + m\log d}$ amortized time per data point.
\qed
\end{proof}

\section{Experimental evaluation}\label{sec:exp}
\newlength{\subfigh}
\newlength{\subfigw}

In this section we present our experimental evaluation.

\textbf{Setup}: We implemented the algorithm with C++,\!\footnote{The code is available at \url{https://version.helsinki.fi/DACS}.}
and conducted the experiments using a 2.3GHz Intel Core i5 processor and 16GB RAM.
Unless specified, we set $\alpha = 0.1$.
We used \algbasesparse and \algbasegini as baselines, algorithms
that select the optimal feature in $\bigO{d}$ time by testing every seen feature.

\textbf{Datasets}:
In order to test the algorithms we generated a family of datasets.
A dataset in 
the family, named $\dtsparse(n, d_1, d_2, q)$, consisted of $n$ data points and $d_1 + d_2$ features.
Before generating the data,
we sample $\theta_j$ from $[0, 1]$ uniformly for each $j = 1, \ldots d_1$.
To generate a data point we first sample
a label, say $y_i$, from a Bernoulli distribution $\bern{1/2}$.
The $j$th feature, where $j = 1, \ldots, d_1$ was set to $y_i$, and flipped with a probability $\theta_j$.
The remaining features are generated independently from a Bernoulli distribution $\bern{q}$.
Unless specified, we set $n = 10\,000$, $d_1 = 10$, $d_2 = 10\,000$ and $q = 10 / 10\,000$.

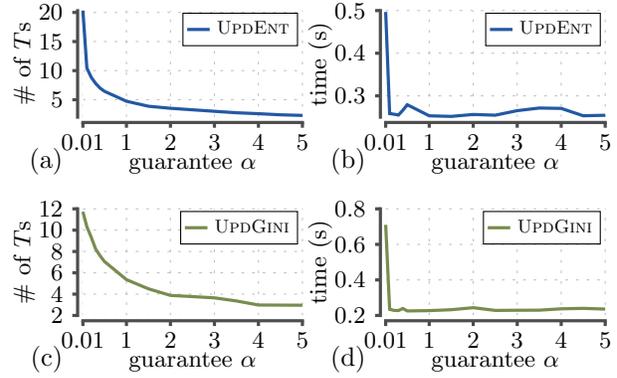
\begin{figure}[t!]
\begin{subcaptiongroup}
\phantomcaption\label{fig:sparsecntvsalpha}
\phantomcaption\label{fig:sparsetimevsalpha}
\phantomcaption\label{fig:ginicntvsalpha}
\phantomcaption\label{fig:ginitimevsalpha}
\setlength{\subfigh}{3cm}
\setlength{\subfigw}{4.5cm}
\setlength{\tabcolsep}{0pt}
\begin{tabular}{ll}
\begin{tikzpicture}
\begin{axis}[xlabel={guarantee $\alpha$}, ylabel= {\# of $T$s},
    width = \subfigw,
    height = \subfigh,
    ymin = 2,
    scaled y ticks = false,
    cycle list name=yaf,
    yticklabel style={/pgf/number format/fixed},
    xticklabel style={/pgf/number format/fixed},
	xtick = {0.01, 1, 2, 3, 4, 5},
    no markers,
	legend entries = {\algupdsparse}
]
\addplot table [x=alpha, y=bcnt] {sparse_approx.txt};
\pgfplotsextra{\yafdrawaxis{0.01}{5}{2}{20.3}}
\end{axis}
\node[anchor=north east] at (-4pt, -8pt) {(a)};
\end{tikzpicture} &
\begin{tikzpicture}
\begin{axis}[xlabel={guarantee $\alpha$}, ylabel= {time (s)},
    width = \subfigw,
    height = \subfigh,
    ymin = 0.25,
    ymax = 0.50,
	xmin = 0.01,
    scaled y ticks = false,
    cycle list name=yaf,
    yticklabel style={/pgf/number format/fixed},
    xticklabel style={/pgf/number format/fixed},
	xtick = {0.01, 1, 2, 3, 4, 5},
    no markers,
	legend entries = {\algupdsparse}
]
\addplot table [x=alpha, y=time1] {sparse_approx.txt};
\pgfplotsextra{\yafdrawaxis{0.01}{5}{0.25}{0.5}}
\end{axis}
\node[anchor=north east] at (-4pt, -8pt) {(b)};
\end{tikzpicture} \\
\begin{tikzpicture}
\begin{axis}[xlabel={guarantee $\alpha$}, ylabel= {\# of $T$s},
    width = \subfigw,
    height = \subfigh,
    ymin = 2,
    ymax = 12,
    scaled y ticks = false,
    cycle list name=yaf,
    yticklabel style={/pgf/number format/fixed},
    xticklabel style={/pgf/number format/fixed},
	xtick = {0.01, 1, 2, 3, 4, 5},
    no markers,
	legend entries = {\algupdgini}
]
\addplot[yafcolor3] table [x=alpha, y=bcnt] {gini_approx.txt};
\pgfplotsextra{\yafdrawaxis{0.01}{5}{2}{12}}
\end{axis}
\node[anchor=north east] at (-4pt, -8pt) {(c)};
\end{tikzpicture} &
\begin{tikzpicture}
\begin{axis}[xlabel={guarantee $\alpha$}, ylabel= {time (s)},
    width = \subfigw,
    height = \subfigh,
    ymin = 0.2,
    ymax = 0.8,
	xmin = 0.01,
    scaled y ticks = false,
    cycle list name=yaf,
    yticklabel style={/pgf/number format/fixed},
    xticklabel style={/pgf/number format/fixed},
	xtick = {0.01, 1, 2, 3, 4, 5},
    no markers,
	legend entries = {\algupdgini}
]
\addplot[yafcolor3] table [x=alpha, y=time1] {gini_approx.txt};
\pgfplotsextra{\yafdrawaxis{0.01}{5}{0.2}{0.8}}
\end{axis}
\node[anchor=north east] at (-4pt, -8pt) {(d)};
\end{tikzpicture}
\end{tabular}

\end{subcaptiongroup}

\caption{Statistics as a function of $\alpha$ for \algupdsparse and \algupdgini:
Figures~\ref{fig:sparsecntvsalpha},\ref{fig:ginicntvsalpha} show the number of non-empty trees,
Figures~\ref{fig:sparsetimevsalpha},\ref{fig:ginitimevsalpha} show the running time.
Figures show averages of 100 repeats.
}
\label{fig:sparsevsalpha}
\end{figure}

\textbf{Results:} 
Let us first consider the role of the approximation guarantee $\alpha$.
Figures~\ref{fig:sparsecntvsalpha},\ref{fig:ginicntvsalpha} show the number of non-empty trees after processing the data
as a function of $\alpha$. The number of trees declines and the decline slows down as $\alpha$
increases. This is in line with Propositions~\ref{prop:sparsebound}~and~\ref{prop:ginitime} stating that the number of trees is in $\bigO{\alpha^{-1} \log \log n}$ and $\bigO{\alpha^{-1}}$ 
for \algupdsparse and \algupdgini, respectively. The number of non-empty trees is less for \algupdgini than for \algupdsparse. 

The total running time of processing each point, shown in Figure~\ref{fig:sparsetimevsalpha},\ref{fig:ginitimevsalpha}, stays relatively constant for $\alpha > 0.1$ but increases as we decrease $\alpha < 0.1$.
For comparison, \algbasesparse required $4.52$s, that is, 17 times slower than $0.26$s needed by \algupdsparse with $\alpha = 0.1$.
Similarly, \algbasegini required $1.69$s whereas \algupdgini required $0.23$s when $\alpha = 0.1$.

We also compared 
the score of the approximate split and the score of the optimal split after the last data point has been processed.
For the tested $\alpha$s the approximate split always returned the optimal split for both objectives.
The approximate split yielded suboptimal results when we set $\alpha$ to an unreasonably high value such as $\alpha = 40$.

\begin{figure}[t!]
\begin{subcaptiongroup}
\phantomcaption\label{fig:sparsetimevscnt}
\phantomcaption\label{fig:sparsecntvscnt}
\phantomcaption\label{fig:ginitimevscnt}
\setlength{\subfigh}{3cm}
\setlength{\subfigw}{3.4cm}
\begin{tikzpicture}
\begin{axis}[xlabel={$n / 1000$}, ylabel= {time (s)},
    width = \subfigw,
    height = \subfigh,
    ymin = 0,
    ymax = 5,
    scaled x ticks = false,
    scaled y ticks = false,
    cycle list name=yaf,
    yticklabel style={/pgf/number format/fixed},
	ytick = {0, 1, 2, 3, 4, 5},
    no markers,
]
\addplot table [x expr={\thisrow{cnt}/1000}, y=time1] {sparse_cnt.txt} node[black, pos=0.6, sloped, above, font=\scriptsize, inner sep = 1pt] {\algupdsparse};
\addplot table [x expr={\thisrow{cnt}/1000}, y=time2] {sparse_cnt.txt} node[black, pos=0.5, sloped, above, font=\scriptsize, inner sep = 1pt] {\algbasesparse};
\pgfplotsextra{\yafdrawaxis{0.5}{10}{0}{5}}
\end{axis}
\node[anchor=north east] at (0, -8pt) {(a)};
\end{tikzpicture}%
\begin{tikzpicture}
\begin{axis}[xlabel={$n / 1000$}, ylabel= {\# of $T$s},
    width = \subfigw,
    height = \subfigh,
    ymin = 10,
    ymax = 20,
    scaled x ticks = false,
    scaled y ticks = false,
    cycle list name=yaf,
    yticklabel style={/pgf/number format/fixed},
    no markers,
]
\addplot table [x expr={\thisrow{cnt}/1000}, y=bcnt] {sparse_cnt.txt} node[black, pos=0.4, sloped, above, font=\scriptsize] {\algupdsparse};
\pgfplotsextra{\yafdrawaxis{0.5}{10}{10}{20}}
\end{axis}
\node[anchor=north east] at (0, -8pt) {(b)};
\end{tikzpicture}%
\begin{tikzpicture}
\begin{axis}[xlabel={$n / 1000$}, ylabel= {time (s)},
    width = \subfigw,
    height = \subfigh,
    ymin = 0,
    ymax = 2,
    scaled x ticks = false,
    scaled y ticks = false,
    cycle list name=yaf,
    yticklabel style={/pgf/number format/fixed},
	ytick = {0, 1, 2, 3, 4, 5},
    no markers,
]
\addplot[yafcolor3] table [x expr={\thisrow{cnt}/1000}, y=time1] {gini_cnt.txt} node[black, pos=0.6, sloped, above, font=\scriptsize, inner sep = 1pt] {\algupdgini};
\addplot[yafcolor4] table [x expr={\thisrow{cnt}/1000}, y=time2] {gini_cnt.txt} node[black, pos=0.5, sloped, above, font=\scriptsize, inner sep = 1pt] {\algbasegini};
\pgfplotsextra{\yafdrawaxis{0.5}{10}{0}{2}}
\end{axis}
\node[anchor=north east] at (0, -8pt) {(c)};
\end{tikzpicture}
\end{subcaptiongroup}

\caption{Statistics as a function of the number of data points in synthetic data:
Figures~\ref{fig:sparsetimevscnt},\ref{fig:ginitimevscnt} show the running times and
Figure~\ref{fig:sparsecntvscnt} shows the number of non-empty trees. 
Figures show averages of 100 repeats.}

\end{figure}
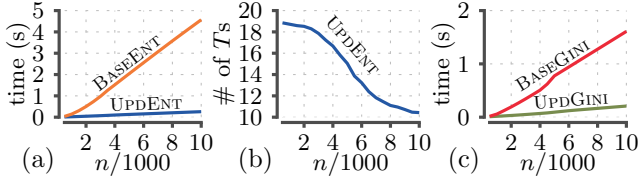

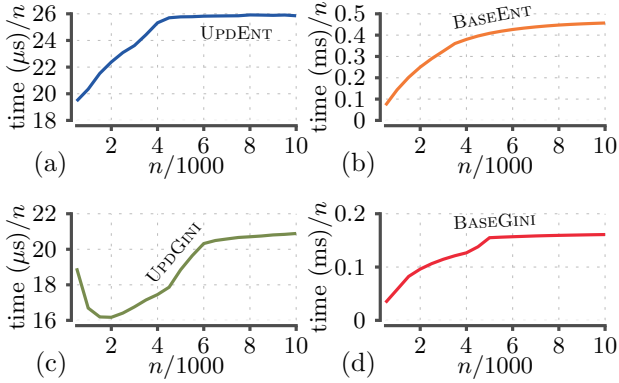
\begin{figure}[t!]
\begin{subcaptiongroup}
\phantomcaption\label{fig:sparsetimeratioapprox}
\phantomcaption\label{fig:sparsetimeratioexact}
\phantomcaption\label{fig:ginitimeratioapprox}
\phantomcaption\label{fig:ginitimeratioexact}
\setlength{\subfigh}{3cm}
\setlength{\subfigw}{4.5cm}
\setlength{\tabcolsep}{0pt}
\begin{tabular}{ll}
\begin{tikzpicture}
\begin{axis}[xlabel={$n / 1000$}, ylabel= {$\text{time ($\mu$s)} / n$},
    width = \subfigw,
    height = \subfigh,
    ymin = 18,
    ymax = 26,
    scaled x ticks = false,
    scaled y ticks = false,
    cycle list name=yaf,
    yticklabel style={/pgf/number format/fixed, /pgf/number format/precision=3},
    no markers,
]
\addplot table [x expr={\thisrow{cnt}/1000}, y expr={\thisrow{tr1}*1000}] {sparse_cnt.txt} node[black, pos=0.8, sloped, below, font=\scriptsize] {\algupdsparse};
\pgfplotsextra{\yafdrawaxis{0.5}{10}{18}{26}}
\end{axis}
\node[anchor=north east] at (0, -8pt) {(a)};
\end{tikzpicture} &
\begin{tikzpicture}
\begin{axis}[xlabel={$n / 1000$}, ylabel= {$\text{time (ms)} / n$},
    width = \subfigw,
    height = \subfigh,
    ymin = 0,
    ymax = 0.5,
    scaled x ticks = false,
    scaled y ticks = false,
    cycle list name=yaf,
    yticklabel style={/pgf/number format/fixed},
	ytick = {0, 0.1, 0.2, 0.3, 0.4, 0.5},
    no markers,
]
\addplot[yafcolor2] table [x expr={\thisrow{cnt}/1000}, y=tr2] {sparse_cnt.txt} node[black, pos=0.5, sloped, above, font=\scriptsize] {\algbasesparse};
\pgfplotsextra{\yafdrawaxis{0.5}{10}{0}{0.5}}
\end{axis}
\node[anchor=north east] at (0, -8pt) {(b)};
\end{tikzpicture} \\
\begin{tikzpicture}
\begin{axis}[xlabel={$n / 1000$}, ylabel= {$\text{time ($\mu$s)} / n$},
    width = \subfigw,
    height = \subfigh,
    ymin = 16,
    ymax = 22,
    scaled x ticks = false,
    scaled y ticks = false,
    cycle list name=yaf,
    yticklabel style={/pgf/number format/fixed, /pgf/number format/precision=3},
    no markers,
]
\addplot[yafcolor3] table [x expr={\thisrow{cnt}/1000}, y expr={\thisrow{tr1}*1000}] {gini_cnt.txt} node[black, pos=0.6, sloped, above, font=\scriptsize] {\algupdgini};
\pgfplotsextra{\yafdrawaxis{0.5}{10}{16}{22}}
\end{axis}
\node[anchor=north east] at (0, -8pt) {(c)};
\end{tikzpicture} &
\begin{tikzpicture}
\begin{axis}[xlabel={$n / 1000$}, ylabel= {$\text{time (ms)} / n$},
    width = \subfigw,
    height = \subfigh,
    ymin = 0,
    ymax = 0.2,
    scaled x ticks = false,
    scaled y ticks = false,
    cycle list name=yaf,
    yticklabel style={/pgf/number format/fixed},
	ytick = {0, 0.1, 0.2, 0.3, 0.4, 0.5},
    no markers,
]
\addplot[yafcolor4] table [x expr={\thisrow{cnt}/1000}, y=tr2] {gini_cnt.txt} node[black, pos=0.5, sloped, above, font=\scriptsize] {\algbasegini};
\pgfplotsextra{\yafdrawaxis{0.5}{10}{0}{0.2}}
\end{axis}
\node[anchor=north east] at (0, -8pt) {(d)};
\end{tikzpicture}
\end{tabular}
\end{subcaptiongroup}

\caption{Running time for processing a single data point as a function of the number of the data points in synthetic data:
Note that $y$-axes have different scales and time units.
Figures show averages of 100 repeats.}
\label{fig:sparsetimeratio}

\end{figure}

Next let us consider the running time as a function of the number of data points.
Figure~\ref{fig:sparsetimevscnt} shows the running time as a function of the processed data points.
We see that both the baseline and the approximation algorithm grow somewhat linearly.
A more detailed picture is given in Figure~\ref{fig:sparsetimeratio} where we see
show the running time for processing an individual data point as a function of the already processed data points.
Generally, all algorithms see increase in processing time. However, this is due to the fact that we ignore
the features that we have not seen, smaller datasets will have more features that have not been active
in a single transaction. Interestingly enough, \algupdgini is slower for early data points and speedup until 2000 points have been seen.
We conjecture that this is due to the increased need to update the search tree for features (Line~\ref{alg:giniupdfeature} in \algupdgini) as
at the beginning the probabilities $c_{j01}/n_{j0}$ are not yet accurate and may change frequently.

In Figure~\ref{fig:sparsecntvscnt} we see that the number of non-empty
trees \emph{decreases} as the number of points increases. This is due to the following reason: in the beginning
the frequencies $c_{i01} / n_{i0}$ fluctuate, leading to many variables being included in many trees.
As the number of data points increases, the frequencies become more stable, leading to fewer calls on Lines~\ref{alg:extend1}~and~\ref{alg:extend2}
in \algupdsparse. On the other hand, the number of trees remains constant, approximately 10, for \algupdgini as the number of points increases.

\begin{figure}[t!]
\begin{subcaptiongroup}
\phantomcaption\label{fig:sparsetimevsdim}
\phantomcaption\label{fig:sparsetimevsdens}
\phantomcaption\label{fig:ginitimevsdim}
\phantomcaption\label{fig:ginitimevsdens}
\setlength{\subfigh}{3cm}
\setlength{\subfigw}{4.8cm}
\setlength{\tabcolsep}{0pt}
\begin{tabular}{ll}
\begin{tikzpicture}
\begin{axis}[xlabel={dim.$/1000$}, ylabel= {time (s)},
    width = \subfigw,
    height = \subfigh,
    ymin = 0,
    ymax = 5,
    scaled x ticks = false,
    scaled y ticks = false,
    cycle list name=yaf,
    yticklabel style={/pgf/number format/fixed},
    xticklabel style={/pgf/number format/fixed},
	ytick = {0, 1, 2, 3, 4, 5},
    no markers,
]
\addplot table [x expr={\thisrow{dim}/1000}, y=time1] {sparse_noise.txt} node[black, pos=0.585, above, font=\scriptsize] {\algupdsparse};
\addplot table [x expr={\thisrow{dim}/1000}, y=time2] {sparse_noise.txt} node[black, pos=0.8, sloped, above, font=\scriptsize] {\algbasesparse};
\pgfplotsextra{\yafdrawaxis{0.5}{10}{0}{5}}
\end{axis}
\node[anchor=north east] at (0, -8pt) {(a)};
\end{tikzpicture} &
\begin{tikzpicture}
\begin{axis}[xlabel={density}, ylabel= {time (s)},
    width = \subfigw,
    height = \subfigh,
    ymin = 0,
    ymax = 5,
	xmin = 0,
	xmax = 0.03,
    scaled x ticks = false,
    scaled y ticks = false,
    cycle list name=yaf,
    yticklabel style={/pgf/number format/fixed},
    xticklabel style={/pgf/number format/fixed},
	ytick = {0, 1, 2, 3, 4, 5},
    no markers,
	legend entries = {\algbasesparse, \algupdsparse}
]
\addplot+[yafcolor2] table [x=dens, y=time2] {sparse_noise.txt}; 
\addplot+[yafcolor5] table [x=dens, y=time1] {sparse_noise.txt};
\pgfplotsextra{\yafdrawaxis{0}{0.03}{0}{5}}
\end{axis}
\node[anchor=north east] at (0, -8pt) {(b)};
\end{tikzpicture} \\
\begin{tikzpicture}
\begin{axis}[xlabel={dim.$/1000$}, ylabel= {time (s)},
    width = \subfigw,
    height = \subfigh,
    ymin = 0,
    ymax = 2,
    scaled x ticks = false,
    scaled y ticks = false,
    cycle list name=yaf,
    yticklabel style={/pgf/number format/fixed},
    xticklabel style={/pgf/number format/fixed},
	ytick = {0, 1, 2, 3, 4, 5},
    no markers,
]
\addplot[yafcolor3] table [x expr={\thisrow{dim}/1000}, y=time1] {gini_noise.txt} node[black, pos=0.6, above, font=\scriptsize] {\algupdgini};
\addplot[yafcolor4] table [x expr={\thisrow{dim}/1000}, y=time2] {gini_noise.txt} node[black, pos=0.7, sloped, above, font=\scriptsize] {\algbasegini};
\pgfplotsextra{\yafdrawaxis{0.5}{10}{0}{2}}
\end{axis}
\node[anchor=north east] at (0, -8pt) {(c)};
\end{tikzpicture} &
\begin{tikzpicture}
\begin{axis}[xlabel={density}, ylabel= {time (s)},
    width = \subfigw,
    height = \subfigh,
    ymin = 0,
    ymax = 2,
	xmin = 0,
	xmax = 0.03,
    scaled x ticks = false,
    scaled y ticks = false,
    cycle list name=yaf,
    yticklabel style={/pgf/number format/fixed},
    xticklabel style={/pgf/number format/fixed},
	ytick = {0, 1, 2, 3, 4, 5},
    no markers,
	legend entries = {\algbasegini, \algupdgini}
]
\addplot+[yafcolor4] table [x=dens, y=time2] {gini_noise.txt}; 
\addplot+[yafcolor3] table [x=dens, y=time1] {gini_noise.txt};
\pgfplotsextra{\yafdrawaxis{0}{0.03}{0}{2}}
\end{axis}
\node[anchor=north east] at (0, -8pt) {(d)};
\end{tikzpicture}
\end{tabular}
\end{subcaptiongroup}

\caption{
Figures~\ref{fig:sparsetimevsdim},\ref{fig:ginitimevsdim} show
the
running time as a function of 
the number of features for data sets $\dtsparse(10000, 10, d, 10 / d)$.
As $d$ increases the number of features increases but the average number of 1s
stays the same. Figures~\ref{fig:sparsetimevsdens},\ref{fig:ginitimevsdens}
show the running as a function of density (=proportion of 1s) for the same data sets.
}

\label{fig:sparsevsdens}

\end{figure}
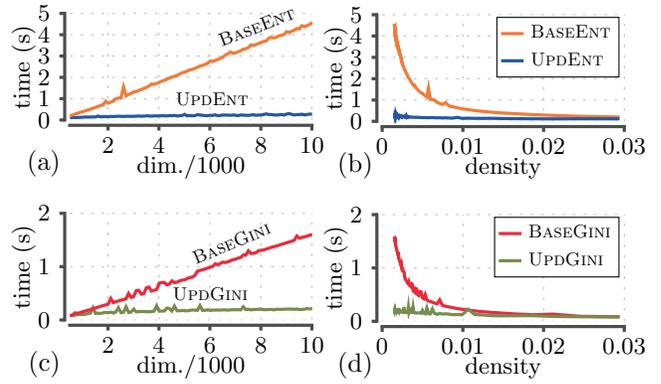

Finally, let us consider the effect on the running time as a function of sparsity.
In order to test this we generated 96 datasets $\dtsparse(10000, 10, d, 10 / d)$ by varying $d = 500, 600, \ldots, 10000$.
Each dataset has the same expected number of 1s but increasingly large number of features. 
We see in Figure~\ref{fig:sparsevsdens} that \algupdsparse and \algupdgini are especially beneficial over the baselines
when dealing with sparse data.

\textbf{Results with benchmark datasets}:
In addition to the synthetic datasets we used benchmark datasets.
\dtname{Author}\footnote{\label{footnote:uci}\url{https://archive.ics.uci.edu/ml/index.php}} is a subset of a dataset depicting authors from Victorian era~\citep{gungor2018benchmarking}, consisting of 100 word text snippets and
the label being the gender of the author.
\dtname{Farm}\footref{footnote:uci} consists of online ads from farm related websites, the label being whether the ad has been approved by the content owner.
\dtname{SRAA}\footnote{\url{https://people.cs.umass.edu/~mccallum/data.html}, \url{https://www.csie.ntu.edu.tw/~cjlin/libsvmtools/datasets/binary.html}\label{footnote:svm}} (also named \dtname{real-sim}) consists of Usenet posts discussing real vs. simulated racing and aviation.
\dtname{20News}\footref{footnote:svm} consists of Usenet posts from 20 groups, where \texttt{sci.*}, \texttt{sci.*}, and \texttt{misc.forsale} are labelled as 1
and the remaining groups are labelled as 0.
The characteristics of the datasets are given in Table~\ref{tab:dt}.

\begin{table}[t!]

\caption{Characteristics of the datasets.
Density reports the proportion of 1s.
}
\label{tab:dt}

\begin{tabular*}{\columnwidth}{l@{\extracolsep{\fill}}rrr}
\toprule
Name & $n$ & $d$ & density \\
\midrule
\dtname{Author} & 53678 & 10000 & 0.007 \\
\dtname{Farm} & 4143 & 54877 & 0.0036 \\
\dtname{SRAA} & 72309 & 20958 & 0.0024 \\
\dtname{20News} & 19996 & 1355191 & 0.00033 \\
\bottomrule
\end{tabular*}
\end{table}

\begin{table}[t!]

\caption{
Running times ($\alpha = 0.1$).
}
\label{tab:results}

\setlength{\tabcolsep}{0pt}
\begin{tabular*}{\columnwidth}{l@{\extracolsep{\fill}}rrrr}
\toprule
Name & \algupdsparse & \algbasesparse & \algupdgini & \algbasegini  \\
\midrule
\dtname{Author} & 
6.71s &
47.57s &
6.41s &
11.49s \\

\dtname{Farm} & 
2.75s &
5.52s &
2.44s &
3.78s \\

\dtname{SRAA} & 
11.52s &
1m18s &
7.86s &
25.37s \\

\dtname{20News} & 
39.35s &
22m7s &
31.51s &
11m46s\\

\bottomrule
\end{tabular*}

\end{table}

\begin{figure}[t!]
\begin{subcaptiongroup}
\phantomcaption\label{fig:timevscntnews20}
\phantomcaption\label{fig:timevscntgininews20}

\setlength{\tabcolsep}{0pt}
\setlength{\subfigh}{3cm}
\setlength{\subfigw}{4.3cm}
\begin{tabular}{rr}
\begin{tikzpicture}
\begin{axis}[xlabel={$n/ 1000$}, ylabel= {time (s)},
    width = \subfigw,
    height = \subfigh,
    ymin = 0,
    scaled x ticks = false,
    scaled y ticks = false,
    cycle list name=yaf,
    yticklabel style={/pgf/number format/fixed},
    no markers,
]
\addplot table [x expr = {\thisrow{cnt} / 1000}, y=ea] {news20_cnt.txt} node[black, pos=0.7, sloped, above, font=\scriptsize] {\algupdsparse};
\addplot table [x expr = {\thisrow{cnt} / 1000}, y=eb] {news20_cnt.txt} node[black, pos=0.5, sloped, above, font=\scriptsize] {\algbasesparse};
\pgfplotsextra{\yafdrawaxis{0}{20}{0}{1320}}
\end{axis}
\node[anchor=north east] at (0, -8pt) {(a)};
\end{tikzpicture} &
\begin{tikzpicture}
\begin{axis}[xlabel={$n/ 1000$}, ylabel= {time (s)},
    width = \subfigw,
    height = \subfigh,
    ymin = 0,
    ymax = 700,
    scaled x ticks = false,
    scaled y ticks = false,
    cycle list name=yaf,
    yticklabel style={/pgf/number format/fixed},
    no markers,
]
\addplot[yafcolor3] table [x expr = {\thisrow{cnt} / 1000}, y=ga] {news20_cnt.txt} node[black, pos=0.7, sloped, above, font=\scriptsize] {\algupdgini};
\addplot[yafcolor4] table [x expr = {\thisrow{cnt} / 1000}, y=gb] {news20_cnt.txt} node[black, pos=0.5, sloped, above, font=\scriptsize] {\algbasegini};
\pgfplotsextra{\yafdrawaxis{0}{20}{0}{700}}
\end{axis}
\node[anchor=north east] at (0, -8pt) {(b)};
\end{tikzpicture}
\end{tabular}
\end{subcaptiongroup}

\caption{Running time as a function of $n$ data points as a function of $n$ for \dtname{20News}.
Here, we set $\alpha = 0.1$.}
\label{fig:timevscntbench}
\end{figure}
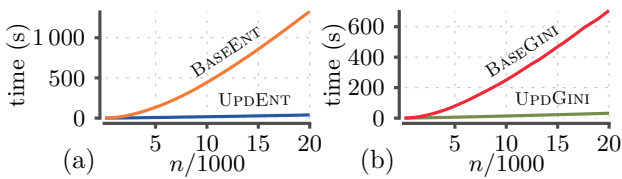

We applied the algorithms to the benchmark datasets by adding each point sequentially,
and querying the split after each addition.

The approximation algorithms returned the optimal features with one exception: for \dtname{SRAA} \algupdgini
returned suboptimal feature for 21 updates (out of 72309). Within these transactions
the largest error occurred when \algupdgini selected a feature with a score of 174.058 whereas the optimal feature
had a score of 173.520, resulting in 1.003 approximation, a significantly lower number than $1 + \alpha = 1.1$ theoretical guarantee.

Let us compare the running times required to process the complete dataset,
The running times given in
Table~\ref{tab:results} show that the approximation algorithms can provide
significant gain.  For example, \algupdsparse required 40 seconds to process
\dtname{20News} while \algbasesparse needed over 22 minutes.  We 
highlight the (cumulative) processing time as a function of added points for
\dtname{20News} in Figure~\ref{fig:timevscntbench}. Similar plots for the remaining benchmark datasets
are given in Appendix (see archived version).
As expected the approximation algorithms are
generally signicantly faster than the baselines. Note that in 
\dtname{20News} (and in \dtname{Farm}) the processing time needed by the baseline algorithm is convex. This is due to the fact that
the number of features in both datasets is large, and the baseline algorithms do
not consider features that have not been seen, resulting in a faster search
when fewer points have been processed. In such cases, the baseline approach may be faster
due to the constant costs of maintaining search trees in the approximation algorithms.
However, as the number of seen features increases over time, the benefit of the approximation algorithms becomes obvious.

\section{Concluding remarks}\label{sec:conclusions}

In this paper we considered estimating splits when constructing decision trees.
We focused on sparse data with binary labels, and
showed that we can estimate the optimal split for sparse data when information gain is used in $\bigO{\alpha^{-1} (1 + m \log d) \log \log n}$ time,
where $n$ is the number of data points, $d$ is the number of features and $m$ is the number of 1s in a data point.
When Gini index is used, we can estimate the split in $\bigO{\alpha^{-1} + m \log d}$ time.
If $m \ll d$, then
this is significantly faster than the $\bigO{d}$ baseline approach.

Speeding up incremental construction of decision trees by approximating splits
yields several potential future lines of work. One of the key assumptions is that the
labels are binary. If this assumption does not hold, then a different appproach is required as adopting the described approach is most likely not feasible.
Moreover, our main focus was time complexity and our algorithms have the same space complexity  as the baseline approaches. Reducing the space
complexity is a potential future line of work.

\bibliographystyle{splncsnat}
\bibliography{bibliography}

\newpage
\appendix
\section{Additional figures}\label{sec:appexp}

We show the
cumulative processing time as a function of added points in
Figure~\ref{fig:timevscntbench2}.

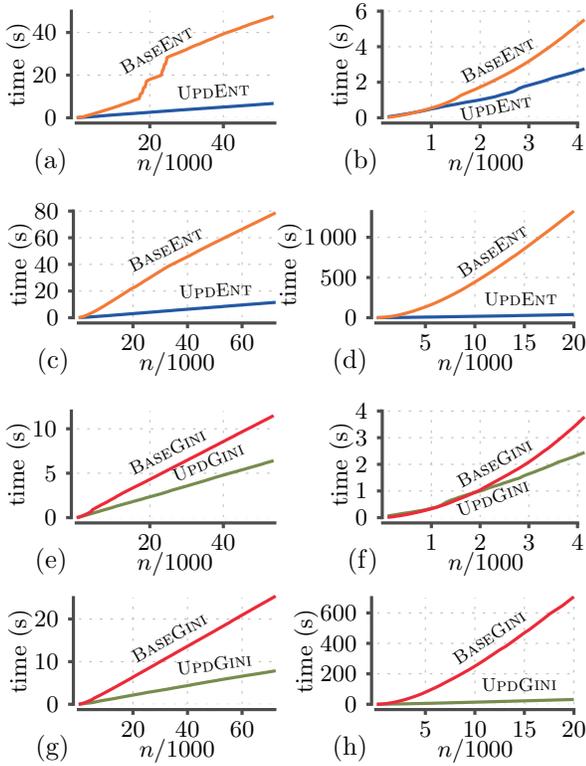
\begin{figure}[h!]
\begin{subcaptiongroup}
\phantomcaption\label{fig:timevscntauthor2}
\phantomcaption\label{fig:timevscntfarm2}
\phantomcaption\label{fig:timevscntrs2}
\phantomcaption\label{fig:timevscntnews202}
\phantomcaption\label{fig:timevscntginiauthor2}
\phantomcaption\label{fig:timevscntginifarm2}
\phantomcaption\label{fig:timevscntginirs2}
\phantomcaption\label{fig:timevscntgininews202}

\setlength{\tabcolsep}{0pt}
\setlength{\subfigh}{3cm}
\setlength{\subfigw}{4.2cm}
\begin{tabular}{rrrr}
\begin{tikzpicture}
\begin{axis}[xlabel={$n/ 1000$}, ylabel= {time (s)},
    width = \subfigw,
    height = \subfigh,
    ymin = 0,
    ymax = 50,
    scaled x ticks = false,
    scaled y ticks = false,
    cycle list name=yaf,
    yticklabel style={/pgf/number format/fixed},
    no markers,
]
\addplot table [x expr = {\thisrow{cnt} / 1000}, y=ea] {author_cnt.txt} node[black, pos=0.7, sloped, above, font=\scriptsize] {\algupdsparse};
\addplot table [x expr = {\thisrow{cnt} / 1000}, y=eb] {author_cnt.txt} node[black, pos=0.5, sloped, above, font=\scriptsize] {\algbasesparse};
\pgfplotsextra{\yafdrawaxis{0}{54}{0}{50}}
\end{axis}
\node[anchor=north east] at (0, -8pt) {(a)};
\end{tikzpicture}&
\begin{tikzpicture}
\begin{axis}[xlabel={$n/ 1000$}, ylabel= {time (s)},
    width = \subfigw,
    height = \subfigh,
    ymin = 0,
    ymax = 6,
    scaled x ticks = false,
    scaled y ticks = false,
    cycle list name=yaf,
    yticklabel style={/pgf/number format/fixed},
    no markers,
]
\addplot table [x expr = {\thisrow{cnt} / 1000}, y=ea] {farm_cnt.txt} node[black, pos=0.5, sloped, below, inner sep=1pt, font=\scriptsize] {\algupdsparse};
\addplot table [x expr = {\thisrow{cnt} / 1000}, y=eb] {farm_cnt.txt} node[black, pos=0.5, sloped, above, font=\scriptsize] {\algbasesparse};
\pgfplotsextra{\yafdrawaxis{0}{4}{0}{6}}
\end{axis}
\node[anchor=north east] at (0, -8pt) {(b)};
\end{tikzpicture}\\
\begin{tikzpicture}
\begin{axis}[xlabel={$n/ 1000$}, ylabel= {time (s)},
    width = \subfigw,
    height = \subfigh,
    ymin = 0,
    ymax = 80,
    scaled x ticks = false,
    scaled y ticks = false,
    cycle list name=yaf,
    yticklabel style={/pgf/number format/fixed},
    no markers,
]
\addplot table [x expr = {\thisrow{cnt} / 1000}, y=ea] {rs_cnt.txt} node[black, pos=0.7, sloped, above, font=\scriptsize] {\algupdsparse};
\addplot table [x expr = {\thisrow{cnt} / 1000}, y=eb] {rs_cnt.txt} node[black, pos=0.5, sloped, above, font=\scriptsize] {\algbasesparse};
\pgfplotsextra{\yafdrawaxis{0}{72}{0}{80}}
\end{axis}
\node[anchor=north east] at (0, -8pt) {(c)};
\end{tikzpicture}&
\begin{tikzpicture}
\begin{axis}[xlabel={$n/ 1000$}, ylabel= {time (s)},
    width = \subfigw,
    height = \subfigh,
    ymin = 0,
    scaled x ticks = false,
    scaled y ticks = false,
    cycle list name=yaf,
    yticklabel style={/pgf/number format/fixed},
    no markers,
]
\addplot table [x expr = {\thisrow{cnt} / 1000}, y=ea] {news20_cnt.txt} node[black, pos=0.7, sloped, above, font=\scriptsize] {\algupdsparse};
\addplot table [x expr = {\thisrow{cnt} / 1000}, y=eb] {news20_cnt.txt} node[black, pos=0.5, sloped, above, font=\scriptsize] {\algbasesparse};
\pgfplotsextra{\yafdrawaxis{0}{20}{0}{1320}}
\end{axis}
\node[anchor=north east] at (0, -8pt) {(d)};
\end{tikzpicture}\\
\begin{tikzpicture}
\begin{axis}[xlabel={$n/ 1000$}, ylabel= {time (s)},
    width = \subfigw,
    height = \subfigh,
    ymin = 0,
    ymax = 12,
    scaled x ticks = false,
    scaled y ticks = false,
    cycle list name=yaf,
    yticklabel style={/pgf/number format/fixed},
    no markers,
]
\addplot[yafcolor3] table [x expr = {\thisrow{cnt} / 1000}, y=ga] {author_cnt.txt} node[black, pos=0.7, sloped, above, font=\scriptsize] {\algupdgini};
\addplot[yafcolor4] table [x expr = {\thisrow{cnt} / 1000}, y=gb] {author_cnt.txt} node[black, pos=0.5, sloped, above, font=\scriptsize] {\algbasegini};
\pgfplotsextra{\yafdrawaxis{0}{54}{0}{12}}
\end{axis}
\node[anchor=north east] at (0, -8pt) {(e)};
\end{tikzpicture}&
\begin{tikzpicture}
\begin{axis}[xlabel={$n/ 1000$}, ylabel= {time (s)},
    width = \subfigw,
    height = \subfigh,
    ymin = 0,
    ymax = 4,
    scaled x ticks = false,
    scaled y ticks = false,
    cycle list name=yaf,
    yticklabel style={/pgf/number format/fixed},
    no markers,
]
\addplot[yafcolor3] table [x expr = {\thisrow{cnt} / 1000}, y=ga] {farm_cnt.txt} node[black, pos=0.5, sloped, below, inner sep=1pt, font=\scriptsize] {\algupdgini};
\addplot[yafcolor4] table [x expr = {\thisrow{cnt} / 1000}, y=gb] {farm_cnt.txt} node[black, pos=0.5, sloped, above, font=\scriptsize] {\algbasegini};
\pgfplotsextra{\yafdrawaxis{0}{4}{0}{4}}
\end{axis}
\node[anchor=north east] at (0, -8pt) {(f)};
\end{tikzpicture}\\
\begin{tikzpicture}
\begin{axis}[xlabel={$n/ 1000$}, ylabel= {time (s)},
    width = \subfigw,
    height = \subfigh,
    ymin = 0,
    ymax = 25,
    scaled x ticks = false,
    scaled y ticks = false,
    cycle list name=yaf,
    yticklabel style={/pgf/number format/fixed},
    no markers,
]
\addplot[yafcolor3] table [x expr = {\thisrow{cnt} / 1000}, y=ga] {rs_cnt.txt} node[black, pos=0.7, sloped, above, font=\scriptsize] {\algupdgini};
\addplot[yafcolor4] table [x expr = {\thisrow{cnt} / 1000}, y=gb] {rs_cnt.txt} node[black, pos=0.5, sloped, above, font=\scriptsize] {\algbasegini};
\pgfplotsextra{\yafdrawaxis{0}{72}{0}{25}}
\end{axis}
\node[anchor=north east] at (0, -8pt) {(g)};
\end{tikzpicture}&
\begin{tikzpicture}
\begin{axis}[xlabel={$n/ 1000$}, ylabel= {time (s)},
    width = \subfigw,
    height = \subfigh,
    ymin = 0,
    ymax = 700,
    scaled x ticks = false,
    scaled y ticks = false,
    cycle list name=yaf,
    yticklabel style={/pgf/number format/fixed},
    no markers,
]
\addplot[yafcolor3] table [x expr = {\thisrow{cnt} / 1000}, y=ga] {news20_cnt.txt} node[black, pos=0.7, sloped, above, font=\scriptsize] {\algupdgini};
\addplot[yafcolor4] table [x expr = {\thisrow{cnt} / 1000}, y=gb] {news20_cnt.txt} node[black, pos=0.5, sloped, above, font=\scriptsize] {\algbasegini};
\pgfplotsextra{\yafdrawaxis{0}{20}{0}{700}}
\end{axis}
\node[anchor=north east] at (0, -8pt) {(h)};
\end{tikzpicture}
\end{tabular}
\end{subcaptiongroup}

\caption{Running time as a function of $n$ data points as a function of $n$ for benchmark datasets.
Fig.~\ref{fig:timevscntauthor2},\ref{fig:timevscntginiauthor2}: \dtname{Author},
Fig.~\ref{fig:timevscntfarm2},\ref{fig:timevscntginifarm2}: \dtname{Farm},
Fig.~\ref{fig:timevscntrs2},\ref{fig:timevscntginirs2}: \dtname{SRAA}, and
Fig.~\ref{fig:timevscntnews202},\ref{fig:timevscntgininews202}: \dtname{20News}. Here, we set $\alpha = 0.1$.}
\label{fig:timevscntbench2}
\end{figure}

\end{document}